\documentclass{article}

\usepackage[preprint]{neurips_2025}

\usepackage{subcaption}
\usepackage{graphicx}
\usepackage[utf8]{inputenc} 
\usepackage[T1]{fontenc}    
\usepackage{hyperref}       
\usepackage{url}            
\usepackage{booktabs}       
\usepackage{amsfonts}       
\usepackage{nicefrac}       
\usepackage{microtype}      
\usepackage{xcolor}         
\usepackage{amsmath,amsthm,amssymb,amsfonts}
\usepackage{bm}
\usepackage{pifont}
\usepackage[normalem]{ulem}
\usepackage{soul}
\usepackage{mathtools}
\usepackage[inline]{enumitem}
\usepackage{graphicx}
\usepackage{natbib}
\usepackage{pdflscape}
\usepackage[makeroom]{cancel}
\usepackage{wrapfig}
\usepackage{amsmath, amssymb}
\usepackage{pgfplots}
\usepackage{pgfplotstable}
\pgfplotsset{compat=1.18}
\usepackage{xcolor}
\usepackage{adjustbox}
\newtheorem{definition}{Definition}
\newcommand{\red}[1]{\textcolor{black}{#1}}
\newcommand{\ii}{\mathbf{i}}

\newtheorem{theorem}{Theorem}

\newcommand{\R}{\mathbb{R}}
\newcommand{\C}{\mathbb{C}}

\newcommand{\F}{\mathcal{F}}

\newcommand{\Deltac}{L^{\widetilde{\mathcal{F}}_{N}}}
\newcommand{\Deltact}{L^{\widetilde{\mathcal{F}}(t)}_{N}}
\newcommand{\Dc}{\widetilde{D}}
\newcommand{\Fc}{\widetilde{\mathcal F}}
\newcommand{\deltac}{\widetilde\delta}
\newcommand{\Lc}{L^{\widetilde{\mathcal{F}}}}
\newcommand{\lname}{Directed Sheaf Neural Network}
\newcommand{\name}{DSNN}
\newcommand{\diag}{\mathrm{diag}}
\DeclareMathOperator{\sgn}{sgn}
\DeclareMathOperator{\sign}{sgn}
\newcommand{\rsqrt}[1]{#1^{-\frac{1}{2}}}
\usepackage{thmtools} 

\title{Sheaves Reloaded: A Directional Awakening}

\author{%
  Stefano Fiorini\thanks{Equal contribution.} \\
  Pattern Analysis \& Computer Vision \\
  Istituto Italiano di Tecnologia, Genoa, Italy \\
  \texttt{s.fiorini1994@gmail.com} \\
  \And
  Hakan Aktas\footnotemark[1] \\
  University of Cambridge, Cambridge, UK \\
  \texttt{hea39@cam.ac.uk} \\
  \And
  Iulia Duta\\
  University of Cambridge, Cambridge, UK\\
  \texttt{id366@cam.ac.uk}
  \And
  Stefano Coniglio\\
    Department of Economics\\
    University of Bergamo, Bergamo, Italy \\
    \And
    Pietro Morerio\\
Pattern Analysis \& Computer Vision\\
Istituto Italiano di Tecnologia, Genoa, Italy
\And
Alessio Del Bue\\
Pattern Analysis \& Computer Vision\\
Istituto Italiano di Tecnologia, Genoa, Italy
\And
Pietro Liò\\
University of Cambridge, Cambridge, UK\\
}

\begin{document}

\maketitle

\begin{abstract}
Sheaf Neural Networks (SNNs) represent a powerful generalization of a Graph Neural Networks (GNNs) that have significantly improved our ability to model complex relational data.
While directionality has been shown to substantially boost the performance of graph-learning tasks and is key to many real-world applications, existing SNNs fall short in representing it.
%
To address this limitation, we introduce the Directed Cellular Sheaf, a special type of cellular sheaf designed to explicitly account for edge orientation.
Building on it, we define a new sheaf Laplacian, the Directed Sheaf Laplacian $\Lc$,  which captures both the graph’s topology and its directional information.
This operator serves as the backbone of the \lname{} (\name), the first SNN model to embed a directional
bias
into its architecture.
Extensive experiments on nine real-world benchmarks show that \name{} consistently outperforms baseline methods.
\end{abstract}

\section{Introduction}


The fast-paced progress in neural networks and deep learning has provided researchers and practitioners with ever more powerful tools for capturing the relationships underlying complex data.
Sheaf Neural Networks (SNNs) have recently emerged as a powerful extension of traditional Graph Neural Networks (GNNs)~\citep{hansen2019toward, bodnar2022neural} 
Such networks rely on the algebraic notion of a {\em cellular sheaf}, which equips a graph with a geometrical structure that assigns vector spaces to nodes/edges and {\em restriction maps} which relate vertex features to edge features, shaping the edge-specific communication. SNNs not only allow working in a higher-dimensional feature space, but also naturally mitigate over-smoothing and improve performance in heterophilic graphs in which neighboring nodes may have dissimilar features~\cite{bodnar2022neural}.

A key limitation of current SNNs is their inability to fully capture the underlying topology of a graph, as they ignore edge orientations and only model undirected connections. For this reason, in this paper, we explore how to incorporate edge directions within the SNN framework in a principled way.
Such a research endeavor is crucial not only to replicate the gains that SNNs have shown over traditional GNNs in the undirected setting, but also to harness the benefits that edge directionality has brought to classical GNNs~\cite{zhang2021magnet} in tasks where data is inherently structured as a directed graph.
Moreover, our work aligns with evidence that edge orientation plays a critical role in complex networks~\citep{bianconi2008local}, as directionality underpins key topological and dynamical phenomena that can profoundly influence a system's behavior~\citep{harush2017dynamic,asllani2018structure}.


To enhance the representational capacity of SNNs in scenarios where directional information is critical, we introduce the concept of {\em Directed Cellular Sheaves}. Unlike traditional cellular sheaves used within state-of-the-art SNNs, which assign vector spaces (or more general algebraic structures) to the cells of a complex without an intrinsic notion of orientation, our framework incorporates directionality directly into the sheaf structure.
Building on this framework, we first define the {\em Directed Coboundary Operator} $\deltac$ associated with the Directed Cellular Sheaf. Subsequently, we rely on $\deltac$ to define the {\em Directed Sheaf Laplacian} (DSL) operator $\Lc$.
This operator not only captures the topological structure of the underlying graph but also faithfully integrates the sign and directionality of its edges.


Our main contributions are summarized as follows:

\begin{itemize}
    \item We introduce the {\em Directed Cellular Sheaf}, a mathematical construct that enriches directed graphs by enabling a principled representation of directional interactions between nodes. This structure assigns linear maps between the data spaces associated with the edges and vertices of a graph, which ensures that information flow between nodes depends on the orientation of the edges.
    \item We propose the {\em Directed Sheaf Neural Network} (DSNN)---an SNN architecture explicitly designed to include an inductive bias that reflects the directional structure of the graph.
    \item We conduct extensive experiments on nine real-world datasets and one synthetic dataset, demonstrating the advantages of our proposed way to incorporate directionality in an SNNs via the Directed Cellular Sheaf and its Laplacian operator $\Lc$.
\end{itemize}

\section{Background \& Related Work}

\subsection{Cellular Sheaves}

In the classical setting, a {\em sheaf} assigns data (such as sets, groups, or vector spaces) to open sets of a topological space (such as points, open segments, and open disks), together with restriction maps that propagate this data to open subsets within them.
A {\em cellular sheaf}~\citep{shepard1985cellular, curry2014sheaves} modifies this perspective by replacing open sets with cells of a cell complex (where 0-cells are points, 1-cells edges, 2-cells faces, etc.).
It assigns a vector space to each cell and a linear restriction map from each higher-dimensional cell to each of its faces, reflecting the hierarchical structure of the complex.
In line with recent works on SNNs~\citep{hansen2019toward, bodnar2022neural}, we focus on cell complexes consisting only of 0-cells and 1-cells, which coincide with the nodes and edges of a graph, and on lower-to-higher dimensional mappings from nodes to edges. In such models, the sheaf structure enables a principled generalization of message-passing architectures by allowing node features to propagate through edge-level transformations governed by linear restriction maps.
 
Following~\cite{hansen2019toward}, we define the \emph{cellular sheaf}
of an undirected graph $G = (V, E)$ with $n=|V|$ and $m=|E|$ as the triple $(\{\F(u)\}_{u \in V}, \{\F(e)\}_{e \in E}, \{\F_{u \trianglelefteq e}\}_{e \in \Gamma(u)})$, containing a vector space $\F(u)$ associated with each vertex $u \in V$, a vector space $\F(e)$ associated with each edge $e \in E$, and a linear map $\F_{u \trianglelefteq e} : \F(u) \rightarrow \F(e)$ for each edge $e \in \Gamma(u)$, where $\Gamma(u)$ is the subset of edges incident on $u$. In line with the SNN literature, all vector spaces are assumed to be real.
In the cellular sheaf, the vector spaces are referred to as \emph{stalks}, while the linear maps are called \emph{restriction maps}.
In this framework, the vertex stalks $\F(u)$ represent the node feature vectors (traditionally denoted as $x_v$ in the graph-learning literature). The space formed by all the spaces associated with the nodes (resp., edges) of the graph is called the space of 0-cochains $C^0(G; \F) = \oplus_{u \in V} \F(u)$ (resp.,
the space of 1-cochains $C^1(G; \F) = \oplus_{e \in E} \F(e)$).
%
The inter-vertex constraints are captured by the \emph{coboundary operator} $\delta : C^0(G; \F) \rightarrow C^1(G; \F)$, which, given an arbitrary orientation on the edges (where, for each $e = \{u,v\} \in E$, either $\F_{u \trianglelefteq e}$ or $\F_{v \trianglelefteq e}$ is multiplied by $-1$), is defined as $\delta (x)_e = \F_{u \trianglelefteq e} \, x_u - \F_{v \trianglelefteq e} \, x_v$.
%
%
From the coboundary operator, one can define the \emph{sheaf Laplacian} as $L^{\F} = \delta^T \delta$ which, for a given $x \in C^0(G; \F)$, reads:
\[
  L^\F(x)_u = \sum_{e = \{u,v\}} \F_{u \unlhd e}^T \left(\F_{u \unlhd e} x_u - \F_{v \unlhd e} x_v \right) \qquad \forall u \in V.
\]
Both $L^\F$ and its normalized version $L^\F_N$ are positive semidefinite operators on the space of 0-cochains $C^0(G; \F)$, and
%
are independent of the chosen edge orientation, mirroring a similar property that haolds for the standard graph Laplacian $L$~\citep{chung1997spectral}.
%
%
%


Several approaches have explored the use of sheaves in the context of graph-based learning. The first SNN was introduced by~\cite{hansen2019toward}, and later extended by~\cite{bodnar2022neural}, who proposed the Neural Sheaf Diffusion (NSD) model.
%
More recent SNN models build upon the NSD framework, incorporating attention mechanisms~\citep{barbero2022sheaf}, extending the architecture to hypergraph data~\citep{duta2023sheaf}, and introducing nonlinearities~\citep{zaghen2024sheaf}.

The SNN literature assumes that all node and edge stalks are finite-dimensional vector spaces of dimension $d$, all of which are isomorphic to $\R^d$.
In this way, every restriction map coincides with a $d\times d$ matrix. As a result, the sheaf Laplacian is a block-matrix of size $nd \times nd$ with blocks of size $d\times d$ which operates on an $nd$-dimensional vector-valued signal obtained by stacking the $d$-dimensional node signals $x_u \in \F(u)$ for all $u \in V$ associated with the graph's vertices (the 0-cochain).
When considering multi-feature vertex signals with $f > 1$ features (or channels), a SNN operates on a matrix-valued graph signal of size $nd \times f$. For any $u,v \in V$, the block of indices $u,v$ of $L^F$ is equal to the $d\times d$ matrix $-\F^T_{u \unlhd e} \F_{v \unlhd e}$.
The \emph{sheaf Laplacian} generalizes the classical graph Laplacian on an undirected and unweighted graph $G$. This is because, in the special case of a \emph{trivial sheaf}---a sheaf where each stalk is isomorphic to $\mathbb{R}$ and each restriction map is the identity map over $\mathbb{R}$---we recover the standard $n \times n$ graph Laplacian $L = D - A$, where $A \in \{0,1\}^{n\times n}$ is the adjacency matrix, and $D := \diag(\mathbf{1}_n^\top A)$ where $\mathbf{1}_n$ is the all-one vector.

To the best of our knowledge, no SNNs, including those introduced in the above-mentioned papers, have been proposed to incorporate the edge directions directly. We set out to do so in this paper.

\subsection{Discrete Laplacian matrices for undirected and directed graphs}

In the literature, GNNs are typically classified into two categories: spectral-based and spatial-based~\citep{wu2020comprehensive}. Spatial-based GNNs define the convolution as a localized-aggregation/message-passing operator~\citep{wang2019dynamic}.
For example, GatedGCN~\citep{li2016gated} handles directed graphs by aggregating information from out-neighbors (ignoring, though, potentially valuable signals from in-neighbors) and, more recently, Dir-GNN~\citep{rossi2024edge} employs separate aggregation schemes with distinct weights for in-neighbors and out-neighbors.
In contrast, spectral-based GNNs define the convolution operator rigorously as a function of the eigenvalue decomposition of the graph Laplacian~\citep{kipf2016semi}. Over the past few years, several approaches have been proposed to generalize spectral convolutions to directed graphs. In particular, DGCN~\citep{tong2020directed} introduces a first-order proximity matrix along with two second-order proximity matrices to describe both the neighborhood of each vertex and the vertices that are reachable from a given vertex in one hop. DiGCN~\citep{Tong2020} adopts the Personalized PageRank matrix
and incorporates $k$-hop diffusion matrices. Finally, several methods generalized the classical Laplacian matrix $L$ to suitably defined complex-valued, Hermitian matrices such as the Magnetic Laplacian~\citep{lieb1993fluxes} and the Sign-Magnetic Laplacian~\citep{fiorini2023sigmanet}.

%
The {\em Magnetic Laplacian} $L^{(q)}$, originally introduced by~\cite{lieb1993fluxes} in the study of electromagnetic fields and later employed in spectral GNNs by~\cite{zhang2021magnet, zhang2021smgc}, is a complex-valued Hermitian matrix that captures directional information in graphs while admitting an eigenvalue decomposition with a real, nonnegative spectrum. 
Letting $A_s := \frac{1}{2} \left(A+A^\top\right)$ be the symmetrized version of $A$ and letting $D_s := \diag(\mathbf{1}_n^\top \, A_s)$, the {\em Magnetic Laplacian} and its normalized version are defined as follows:
\begin{equation*}
 L^{(q)} := D_s - H^{(q)} \text{ and } L^{(q)}_N := I - \rsqrt{D_s} H^{(q)} \rsqrt{D_s}, \text{with } H^{(q)} := A_s \odot \exp \left(\ii \, 2 \pi q\left(A-A^\top \right) \right),
\end{equation*}
where $\ii$ is the imaginary unit and $q \in [0,1]$.

The {\em Sign-Magnetic Laplacian} $L^{\sigma}$, introduced by \cite{fiorini2023sigmanet}, is a Hermitian matrix that is well-defined even for graphs with negative edge weights and possesses several additional desirable properties.
When $q = \frac{1}{4}$, $L^\sigma$ and $L^{(q)}$ coincide if the latter is first computed on the unweighted version of the graph and then element-wise multiplied by $A_s$. Thus, $L^\sigma$ is invariant to a positive weight scaling which could otherwise alter the sign pattern of $L^{(q)}$ and, thus, the edge direction.
Letting $\bar D_s := \diag( \mathbf{1}_n^\top|A_s| )$
and $\sign: \mathbb{R}^{n\times n} \rightarrow \{-1,0,1\}^{n\times n}$ be the component-wise {\em signum} function, $L^{\sigma}$ and its normalized version are defined as follows:
\small
\begin{equation*}
    L^{\sigma} := \bar D_s - H^{\sigma} \text{ and } L^{\sigma}_N := I - \rsqrt{\bar D_s} H^{\sigma} \rsqrt{\bar D_s}, \text{ with } H^{\sigma} := A_s \odot \Big(  e^\top  - \sgn (|A - A^\top|) + \ii \sgn \big(|A| - |A^\top| \big) \Big).
\end{equation*}
\normalsize

\section{Directed Cellular Sheaf, Directed Sheaf Laplacian and Its Properties}


In this paper, we introduce the notion of a {\em Directed Cellular Sheaf}, a special type of cellular sheaf where the node and edge stalks are vector spaces defined over the complex field and in which, assuming finite-dimensional vector spaces, the restriction maps are either real-valued or complex-valued matrices where the latter encode the graph's direction.

\subsection{Directed Cellular Sheaf}

We introduce the Directed Cellular Sheaf for the case of finite-dimensional stalks. This is done solely for the ease of notation, as that definition can be directly extended to the infinite-dimensional case.
%
\begin{definition}
The \emph{Directed Cellular Sheaf}
of a directed graph $G = (V, E)$ with adjacency matrix $A \in \{0,1\}^{n\times n}$ is the tuple $(T^{(q)}, \{\Fc(u)\}_{u \in V}, \{\Fc(e)\}_{e \in E}, \{\Fc_{u \trianglelefteq e}\}_{e \in \Gamma(u)})$
consisting of:
\begin{enumerate}
\item A directional and topological Hermitian matrix $T^{(q)} := \exp(\ii \, 2 \pi q\left(A-A^\top \right))$, with $q \in \R$.
\item A vector space $\Fc(u) \in \C^{d}$ associated with each vertex $u \in V$;
\item A vector space $\Fc(e) \in \C^{d}$ associated with each edge $e \in E$;
\item Two linear maps $\Fc_{u \trianglelefteq e}, \Fc_{v \trianglelefteq e}$ that map $\Fc(u), \Fc(v)$ to $\Fc(e)$ for each edge $e \in E$ with $u \sim_e v$ where 
   $\Fc_{u \trianglelefteq e} \in \mathbb{R}^{d \times d}$ and $\Fc_{v \trianglelefteq e} = \Fc_{v \trianglelefteq e}^0 T_{uv}^{(q)} \in \mathbb{C}^{d \times d}$, with $\Fc_{v \trianglelefteq e}^0 \in \mathbb{R}^{d \times d}$, 
   \end{enumerate}
\end{definition}
where $u \sim_e v$ indicates that $e$ is incident to both $u$ and $v$ regardless of whether it is directed or not. An illustration is provided in Figure~\ref{fig:enter-label}.

\begin{figure}[htbp]
    \centering
    \begin{subfigure}{0.48\textwidth}   \centering\includegraphics[width=\linewidth]{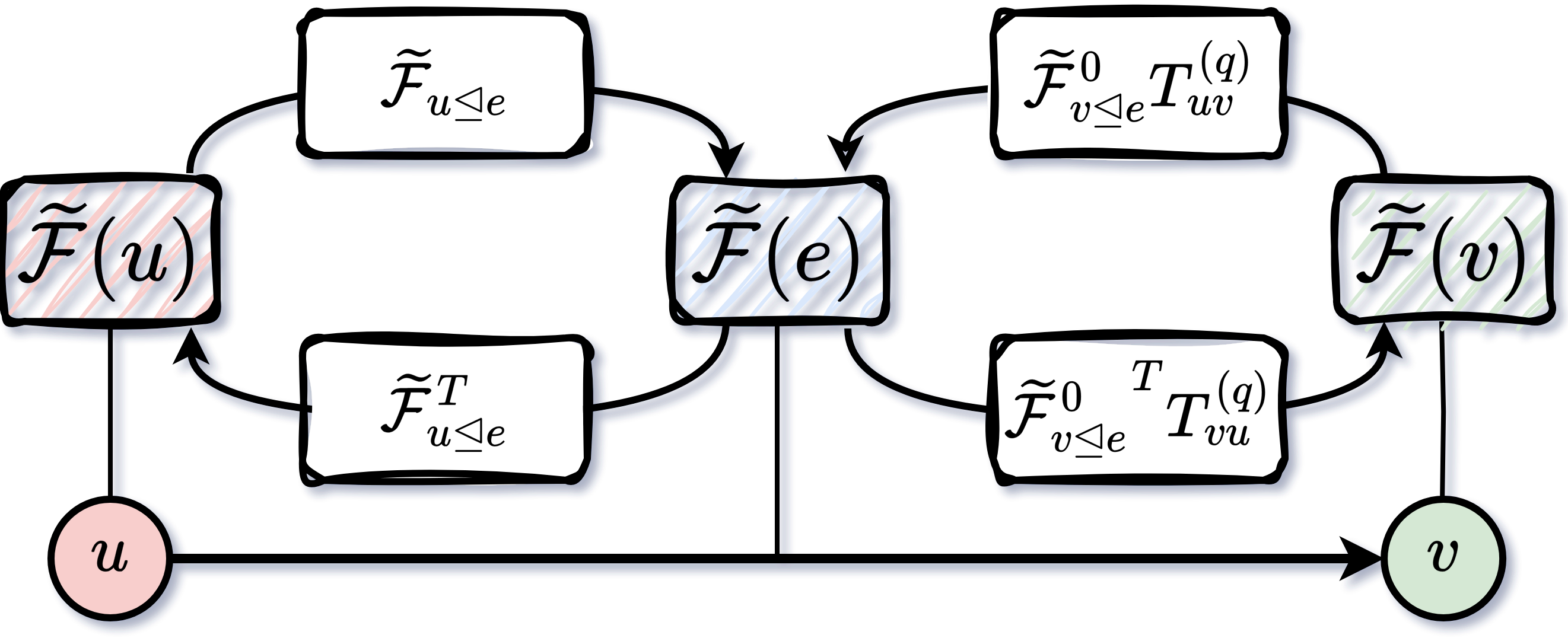}
        \caption{}
        \label{fig:image1}
    \end{subfigure}
    \hfill
    \begin{subfigure}{0.48\textwidth}
        \centering
        \includegraphics[width=\linewidth]{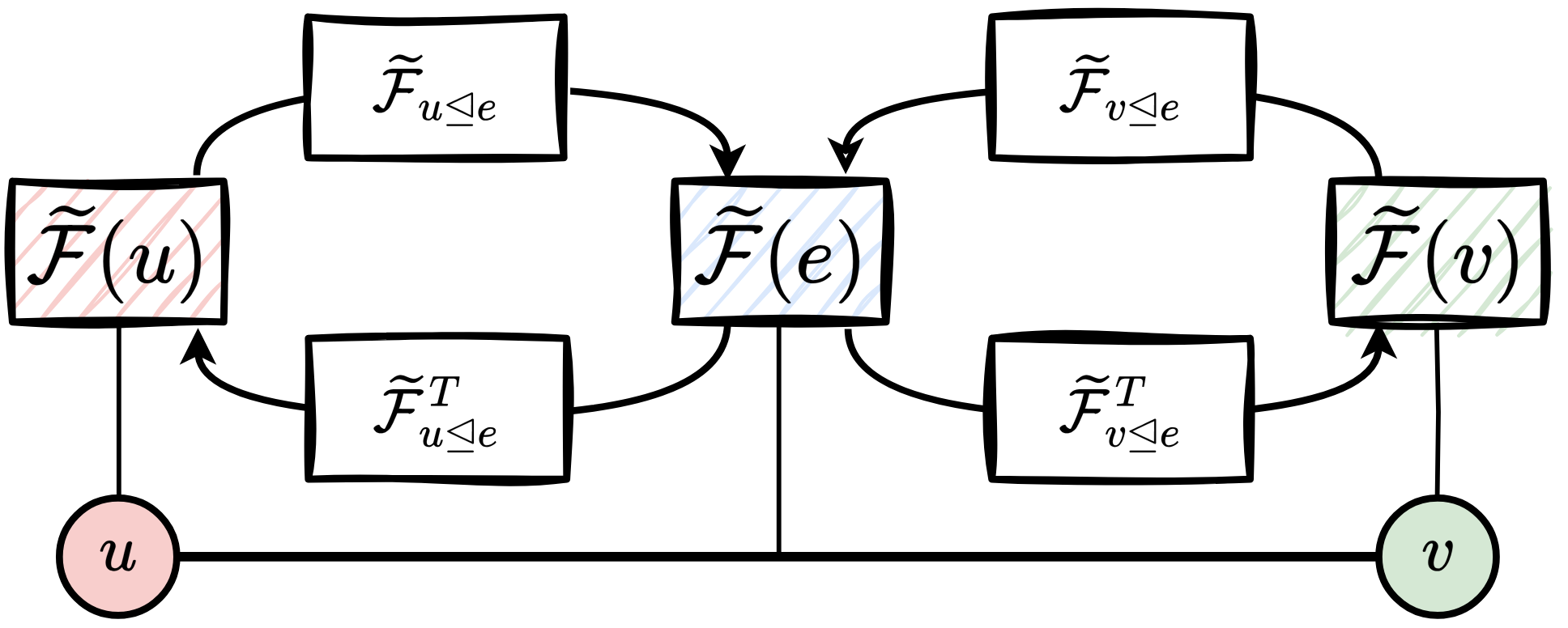}
        \caption{}
        \label{fig:image2}
    \end{subfigure}
    \caption{An illustration of the complex-valued restriction maps of the \textit{Directed Cellular Sheaf} showing how they encode the graph's directionality for (a) a directed edge and (b) an undirected edge.}
    \label{fig:enter-label}
\end{figure}

%
The rationale of such a definition is to encode the direction of each edge in the imaginary part of the restriction map of the tail node in the following sense. $i)$ In the undirected case where $e = \{u,v\}$, $A_{uv} = A_{vu} =1$, $T^{(q)}_{uv} = \cos(0) + \ii \sin(0) = 1$ and, thus, $\Fc_{v \trianglelefteq e} = \Fc_{v \trianglelefteq e}^0$.
$(ii)$ Assume $q=\frac{1}{4}$; in the directed case where $e = (u,v)$, $A_{uv}= 1$ and $A_{vu} = 0$; we have
%
$T^{(q)}_{uv} = \cos(-\pi\frac{1}{2}) + \ii \sin(-\pi\frac{1}{2}) = -\ii$, and thus $\Fc_{v \trianglelefteq e} = -\Fc_{v \trianglelefteq e}^0\ii$. 
Notice that our proposed Directed Cellular Sheaf generalizes the classical Cellular Sheaf since, if $G$ is an undirected graph, $T_{vu}^{(q)} = 1$ for all $\{u,v\} \in E$ for any choice of $q$ and, thus, the two coincide. If $G$ is directed, but we set $q=0$, we obtain the classical Cellular Sheaf associated with the undirected version of $G$.

Let $E^0 \cup E^1 = E$ be a partition of the edge set $E$ into undirected edges ($E^0$) and directed edges ($E^1$).
We define the {\em Directed Coboundary Operator} $\deltac$ associated with the Directed Cellular Sheaf as $\deltac(x)_e = \Fc_{u \unlhd e} \, x_u - \Fc_{v \unlhd e} \, x_v$ for all $e \in E$, where $x$ is a co-chain of the Directed Cellular Sheaf.
Thanks to our definition of $\Fc_{u \trianglelefteq e}, \Fc_{v \trianglelefteq e}$, we have:
\begin{equation}
  \deltac(x)_e = \left\{\begin{array}{ll}
                          \Fc_{u \unlhd e} \, x_u - \Fc_{v \unlhd e} \, x_v & \text{ if } e \in E^0\\
                          \Fc_{u \unlhd e} \, x_u - \Fc_{v \unlhd e}^0 T_{\red{uv}}^{(q)} \, x_v & \text{ if } e \in E^1.
                        \end{array}.
    \right.
\end{equation}
We define the {\em Directed Sheaf Laplacian} (DSL) $\Lc$ associated with a Directed Cellular Sheaf
as $\Lc := \deltac^* \deltac$,
where $*$ is the conjugate transpose operator.
Each $d\times d$ block of $\Lc$ of index $u,v \in V$ is:
\begin{align}\label{eq:Lc:uv}
\Lc_{uv} =& \left\{\begin{array}{ll}
                    -\Fc_{u \unlhd e}^* \Fc_{v \unlhd e} = -\Fc_{u \unlhd e}^T \Fc_{v \unlhd e}^0 T_{uv}^{(q)} & \text{ if } e=(u,v)\\
                    -\Fc_{u \unlhd e}^* \Fc_{v \unlhd e} = -(\Fc_{u \unlhd e}^0 T_{vu}^{(q)})^* \Fc_{v \unlhd e}  & \text{ if } e=(v,u)\\
                    -\Fc_{u \unlhd e}^* \Fc_{v \unlhd e} = -\Fc_{u \unlhd e}^T \Fc_{v \unlhd e} & \text{ if } e=\{u,v\}\\
                    0 & \text{ otherwise}
\end{array}
  \right.\\
  \label{eq:Lc:uu}
  \Lc_{uu} =& \sum_{e \in \Gamma(u)} \Fc_{u \trianglelefteq e}^* \Fc_{u \trianglelefteq e}
  ,
\end{align}
where
$\Gamma(u)$ is the set of edges incident to $u$ regardless of their direction. Notice that, since $(T_{vu}^{(q)})^* = -T_{uv}^{(q)}$, for a directed edge $e=(u,v)$ or $e=(v,e)$, $\Lc_{uv}$ and $\Lc_{vu}$ only differ by the sign of their imaginary part.



%
As one can see (the full derivation is reported in the appendix), when applied to a 0-cochain $x$, the Directed Sheaf Laplacian operator reads as follows for each $u \in V$:
{
\begin{align}\label{eq:Lc}
    \Lc(x)_u = &
    \underbrace{\sum_{e=(v,u) \in E} (\Fc_{u \trianglelefteq e}^0 T_{vu}^{(q)})^* (\Fc_{u \trianglelefteq e} x_u - \Fc_{v \trianglelefteq e} x_v)}_{\text{inflow}}\\
    \notag
    +& \underbrace{\sum_{e=(u,v) \in E} \Fc_{u \trianglelefteq e}^T (\Fc_{u \trianglelefteq e} x_u - \Fc_{v \trianglelefteq e}^0 T_{uv}^{(q)}x_v)}_{\text{outflow}}
    +\underbrace{\sum_{e=\{u,v\} \in E} \Fc_{u \unlhd e}^T (\Fc_{u \unlhd e} x_u -  \Fc_{v \trianglelefteq e} x_v)}_{\text{undirected}}. 
\end{align}
}

We define the {\em normalized Directed Sheaf Laplacian} as:
\begin{equation}
    \Lc_N := \Dc^{-\frac{1}{2}} \Lc \Dc^{-\frac{1}{2}},
\end{equation}
where $\Dc := \diag(\Dc_1, \Dc_2, \dots, \Dc_n)$ and, for all $u \in V$, $\Dc_u := \sum_{e \in \Gamma(u)} \Fc_{u \trianglelefteq e}^* \Fc_{u \trianglelefteq e}$.

\subsection{Spectral properties of the Directed Sheaf Laplacian}

The Directed Sheaf Laplacian enjoys several key spectral properties, which we now illustrate. The proofs of the theorems of this section and the next can be found in Appendix~\ref{appx:proof}.
First, we show that both $\Lc$ and $\Lc_N$ are diagonalizable with a real spectrum and that their spectra are nonnegative:
\begin{theorem}
      $\Lc$ is Hermitian and $\Lc \succeq 0$, and the same holds for $\Lc_N$.
\end{theorem}

Next, we show that the spectrum of the Normalized Sheaf Laplacian is upper-bounded by 2:
\begin{theorem}
  $\Deltac \preceq 2 I$.
\end{theorem}
These theorems show that $\Lc$ and $\Lc_N$ enjoy the same spectral properties as the classical Laplacian matrix $L$ defined for undirected graphs. These are the essential to define a principled convolutional operator which coincides with approximating the graph-Fourier transform of a graph signal with Chebyshev polynomials of the first kind of degree 1, as proposed by~\cite{kipf2016semi} for the undirected case.



\subsection{Generalization properties of the Directed Sheaf Laplacian}

First, we show that the Directed Sheaf Laplacian generalizes both the Sheaf Laplacian and the classical graph Laplacian:
\begin{theorem}\label{thm:undirected}
   If $G$ is undirected, $\Lc$ coincides with the classical sheaf Laplacian $L^\F$ for any choice of $q \in \R$. Also, if the sheaf is trivial and $G$ is undirected and unweighted, $\Lc$ coincides with the classical graph Laplacian $L$. If $G$ is directed and we set $q = 0$, $\Lc$ coincides with the classical sheaf Laplacian associated with the undirected version of $G$.
\end{theorem}

Let a {\em Trivial Directed Cellular Sheaf} be any Directed Cellular Sheaf with $d=1$ where, for all directed edges $e=(u,v)$, $\Fc_{u \trianglelefteq e} = 1$ and $\Fc_{v \trianglelefteq e} = T_{uv}^{(q)}$.
With the next theorem, we show that, for a given directed graph without weights, $\Lc$ generalized the Magnetic Laplacian and, when choosing $q = \frac{1}{4}$, also the Sign-Magnetic Laplacian.
The following holds:
\begin{theorem}
  Letting $G$ be a directed graph with unit weights, the Directed Sheaf Laplacian $\Lc$ associated with a Trivial Directed Cellular Sheaf coincides with the Magnetic Laplacian $L^{(q)}$. In the special case where $q=\frac{1}{4}$, $\Lc$ also coincides with the Sign-Magnetic Laplacian $L^\sigma$.
\end{theorem}

It is well-known that the classical Laplacian matrix $L$ defined for an undirected graph can be equivalently defined as $L=D-A$ or $L = BB^T$, where $B\in\{-1,0,1\}^{n \times m}$ is the node-to-edge incidence matrix of the graph in which either of the two entries of each column has been arbitrarily multiplied by $-1$.
While, to the best of our knowledge, no similar construction is known for the Magnetic Laplacian and the Sign-Magnetic Laplacian, with the following theorem, we show that one such decomposition exists and can be obtained via the lens of our Directed Sheaf Laplacian, thanks to its generality. Indeed, we have the following:
\begin{theorem}
    Let $G$ be a directed graph with unit weights. \red{Assuming a Trivial Directed Cellular Sheaf}, the conjugate transpose $\deltac^*$ of the Directed Coboundary Operator $\deltac$ boils down to the complex-valued node-to-edge incidence matrix $\widehat B \in \mathbb{C}^{\red{n \times m}}$ defined for an edge $e \in E$ incident to a vertex $u$:
  \begin{equation*}
    \widehat B_{\red{ue}} = \left\{\begin{array}{ll}
                    1 & \text{ if } e=(u,v) \text{ or } e=\{u,v\} \text{ with } \red{u < v}\\
                    \red{-1} & \text{ if } e=\{u,v\} \text{ with } \red{u > v}\\
                    -T_{uv}^{(q)} & \text{ if } e=(v,u).
                  \end{array}
  \right.
  \end{equation*}
  It follows that $L^{(q)} = \widehat B \widehat B^*$.
  With $q=\frac{1}{4}$, $L^{(\frac{1}{4})} = L^{\sigma} = \widehat B \widehat B^*$.
\end{theorem}
Incidentally, this result also allows to obtain substantially simpler proofs of the positive semidefiniteness of both Laplacian matrices than those reported in their original papers.

\section{The Directed Sheaf Neural Network (DSNN)}


%
%
The \emph{sheaf diffusion} process on a graph $G$ is introduced in~\cite{hansen2020sheaf} as a generalization of the classical heat diffusion process that governs classical spectral-based GNNs~\cite{kipf2016semi}. It follows the differential equation
\[
\dot{X}(t) = -L_N^{\mathcal{F}} X(t),
\]
where $X(t)$ is a time-dependent graph signal $X$.
More precisely, $X_u$ is the stalk of each node $u \in V$, and it coincides with a matrix in $\mathbb{R}^{d \times f}$, where $d$ denotes the dimensionality of the vertex stalk and $f$ is the number of feature channels.
$X$ is typically obtained starting from a matrix of node features of size $n \times f$ to which one applies a linear projection to obtain an $n \times (df)$ matrix, which is then reshaped to $(nd) \times f$.

%

%
%
By relying on our proposed Directed Sheaf Laplacian $\Lc$, we introduce the {\em Directed Neural Sheaf Diffusion} process as the following generalization of the Neural Sheaf Diffusion process proposed by~\cite{bodnar2022neural}:
\begin{equation}\label{eq:contiheat}
  \dot{X}(t) = -\sigma\left( \Lc_N(t) \Big( I_n \otimes W_1(t) \Big) X(t) W_2(t) \right),
\end{equation}
where $W_1 \in \R^{d \times d}$, $W_2\in\R^{f \times f}$ are two time-dependent weight matrices
and $\sigma$ is a nonlinear activation function.

We then define the \lname{} (\name) as the convolutional neural network whose convolution operator is obtained from the discretized version of
Equation~\ref{eq:contiheat},
which leads to:
\begin{equation}\label{eq:layerwise}
    X^{(t+1)} = X^{(t)}- \sigma\left(\Deltact\left( I_n  \otimes W_1^{(t)} \right) X^{(t)} W_2^{(t)}\right),
\end{equation}
where $X^{(t)}, X^{(t+1)} \in \C^{nd\times f}$.

The expressiveness of Equation~\eqref{eq:layerwise} is further enhanced by learning a parameter $\epsilon \in [-1, 1]^d$ that allows the model to adjust the relative magnitude of the features in each stalk dimension. This gives the update rule as:
\begin{equation}\label{eq:sheafneural}
     X^{(t+1)} = \diag(1 + \varepsilon)X^{(t)}- \sigma\left(\Deltact\left( I_n  \otimes W_1^{(t)} \right) X^{(t)} W_2^{(t)}\right),
\end{equation}
where $\varepsilon \in [-1, 1]^{nd}$ is obtained by concatenating $\epsilon \; n$ times.
As activation function $ \sigma $, we adopt a complex extension of the \textit{ReLU} function, defined for a given $ z \in \mathbb{C}$, as 
\begin{equation*}
   \sigma(z) = 
\begin{cases}
z & \text{if } \Re(z) \geq 0, \\
0 & \text{otherwise}.
\end{cases} 
\end{equation*}
This choice is consistent with previous work on complex-valued GNNs and HNNs, such as~\citep{zhang2021magnet,fiorini2024let}.

Finally, since our model operates in the complex domain, we project the output of the final layer to the real domain using an \textit{unwind} operation. Given $ X^{(\tau)} \in \mathbb{C}^{n \times c} $, the projection is defined as:
\begin{equation*}
\text{unwind}(X{(\tau)}) = \left( \Re(X{(\tau)}) \,\|\, \Im(X{(\tau)}) \right) \in \mathbb{R}^{n \times 2c},    
\end{equation*}
where $\tau$ is the last convolutional layer of the network, $ \| $ denotes concatenation along the feature dimension, and $c$ is the output dimension.




\paragraph{Learnable Sheaf Laplacian.} 
A key strength of SSNs is their ability to operate over richer structures—sheaves—rather than just the underlying graph. Since multiple sheaf structures can be associated with the same graph, effectively modeling the most suitable one is critical for meaningful representation learning. In our proposed models, the restriction maps are learned end-to-end as a function of the input vertex features.
Specifically, for each edge $e \in E$ with endpoints $u,v \in V$, each $d \times d$ matrix $\F_{u \unlhd e}$ is parameterized as $\F_{u \unlhd e} = \Phi(x_u \,\|\, x_v)$, where $x_v$ and $x_u$ denote the feature vectors of the nodes incident to $e$. The resulting vector is reshaped into a $d \times d$ matrix, thus obtaining the linear restriction map $\F_{u \unlhd e}$.

\paragraph{Connection with Neural Sheaf Diffusion.}  The Neural Sheaf Diffusion process proposed by~\cite{bodnar2022neural} relies on the \textit{Normalized Sheaf Laplacian} $L_N^{\mathcal{F}}$ instead of on our proposed \textit{Directed Sheaf Laplacian} $\Lc_N$ in Equation~\eqref{eq:layerwise}. Since, as shown in Theorem~\ref{thm:undirected}, $L_N^{\mathcal{F}} = \Lc_N$ when the graph in undirected, NSD is a special case of  SNN when the graph is undirected.


\paragraph{Computational Complexity.}
Let $f$ be the number of channels, assumed constant throughout the layers, and let's focus on a single convolutional layer.
In the case of an undirected graph, where all restriction maps are real-valued, the complexity of DSNN is identical to the complexity of NSD, and reads $O\big(n(c^2 + d^3) + m(cd^2 + d^3)\big)$, which, with $d=1$, coincides with the complexity of a classical spectral-based GNNs~\cite{kipf2016semi}, which is $\mathcal{O}(nc^2 + mc)$. In the experiments, we use $d \in \{2,5\}$, which only introduced a small, constant overhead with no asymptotic impact.
For a directed graph, the restriction maps are complex-valued, and thus, the stalks are complex-valued from layer 2 onward. This, though, only leads to an extra multiplicative cost of about 4, which is independent of the graph and size of the network and plays no role in the complexity of DSNN.

For the proof of the theorems in this section, and for additional details on DSNN's inference complexity, please refer to Appendix~\ref{appx:proof} and Appendix~\ref{appx:complexity}, respectively.


\section{Experiments}

We evaluate \name{} against state-of-the-art methods on a diverse set of benchmark datasets spanning from real-world (Section~\ref{sec:real}) to synthetic dataset (Section~\ref{sec:synthetic}).
Following~\citep{bodnar2022neural}, we experiment with three types of $d\times d$ blocks in the Directed Sheaf Laplacian $\Lc$, {\em diagonal}, {\em orthogonal}, and {\em general} matrices, which lead to three variants of \name: Diag-\name, O(d)-\name, Gen-\name.
The tables in this section report the best results in boldface, and the second-best results are underlined.
The datasets and code we used are available on GitHub (see Appendix~\ref{appx:implementation}).
Further details on our experiments are reported in Appendix~\ref{appx:dataset}, \ref{appx:experiment}.

\subsection{Real World Experiments}\label{sec:real}

We measure the performance of our proposed DSNN on the node classification task of predicting the class label of each node. The {\tt Texas}, {\tt Wisconsin}, {\tt Cornell}, and {\tt Film} datasets are taken from~\cite{pei2020geom}, while {\tt Citeseer}, {\tt PubMed}, and {\tt Cora} are sourced from~\citep{yan2022two}. The {\tt Squirrel} and {\tt Chameleon} datasets come from~\cite{platonovcritical}.
Since it is known that GNN-type methods often suffer from poor performance on heterophilic datasets (datasets where neighboring nodes have, on average, different labels), it is crucial to assess the performance of one's method also on graphs of this type. For this reason, the datasets we consider span a wide range of edge homophily coefficients, from as low as 0.11 (highly heterophilic) to as high as 0.81 (highly homophilic), allowing us to assess model performance across varying levels of structural homophily.
Following~\citep{bodnar2022neural}, we split the data 10 times and report the mean accuracy and standard deviation. Each split allocates 48\%/32\%/20\% of the nodes per class for training, validation, and testing, respectively.

As baselines, we use a large set of GNN and SNN models from five categories: \textit{i)} classical GNN models: GCN~\citep{kipf2016semi}, GAT~\citep{veličković2018graph}, GraphSAGE~\citep{hamilton2017inductive}; \textit{ii)} GNN models designed for heterophilic graphs: GGCN~\citep{yan2022two}, Geom-GCN~\cite{pei2020geom}, H2GCN~\citep{zhu2020beyond}, GPRGNN~\citep{chien2021adaptive}, FAGCN~\citep{bo2021beyond}, MixHop~\citep{abu2019mixhop}; \textit{iii)} GNN models that address the oversmoothing problem: GCNII~\citep{chen2020simple}; \textit{iv)} GNN models that incorporate edge directionality: MagNet~\citep{zhang2021magnet}, SigMaNet~\citep{fiorini2023sigmanet}, DirGNN~\citep{rossi2024edge}; \textit{v)} SNN models: NSD~\cite{bodnar2022neural}, and NSD-Compl (a variant of NSD employing complex restriction maps which we introduce in this paper for comparison purposes), with the same Diag-, O(d)-, Gen- variants as DSNN.

\begin{table*}[t]
\centering
\caption{Mean and standard deviation of the accuracy on a collection of real-world graph benchmarks on the node classification task.}

\tiny
\begin{tabular}{l c c c c c c c c c}
\toprule
Model & Texas & Wisconsin & Film & Squirrel & Chameleon & Cornell & Citeseer & Pubmed & Cora \\
\midrule
Homophily Level & 0.11 & 0.21 & 0.22 & 0.22 & 0.23 & 0.30 & 0.74 & 0.80 & 0.81 \\
\# Nodes & 183 & 251 & 7,600 & 2223 & 890 & 183 & 3,327 & 18,717 & 2,708 \\
\# Edges & 295 & 466 & 26,752 & 46,998 & 8,854 & 280 & 4,676 & 44,327 & 5,278 \\
\# Classes & 5 & 5 & 5 & 5 & 5 & 5 & 6 & 3 & 7 \\
\midrule
\textbf{Diag-\name} & \textbf{88.65±4.95} & \textbf{90.20±4.02} & \underline{38.34±1.01} &  45.37±2.21& 46.84±4.03 & \textbf{87.84±5.70} & \underline{79.80±1.49} & \textbf{90.23±0.44} & 87.36±1.41\\
\textbf{O(d)-\name} & \underline{87.57±4.04} & \underline{89.80±3.82}& 37.37±0.98& 44.54±2.26& 45.36±3.29& \underline{87.30±7.26} & 77.28±1.63& 90.05±0.55& 87.30±1.62\\
\textbf{
Gen-\name} & \underline{87.57±5.43} & 89.22±3.31& \textbf{38.40±0.75} & 45.34±1.69 & \textbf{47.16±3.54} & \textbf{87.84±6.86} & \textbf{79.88±1.21} & \underline{90.17±0.44} & 87.58±0.72\\
\midrule
Diag-NSD-Comp & 86.49±5.35 & 89.01±4.81 & 37.84±1.04 & \textbf{45.61±1.62} & 46.47±2.83 & 83.51±6.56 & 77.20±1.43 & 89.74±0.46 & 87.65±1.04 \\
O(d)-NSD-Comp & 87.29±5.54 & 89.21±4.57 & 37.06±0.96 & 41.51±2.01 & 45.56±4.18 & 84.86±0.53 & 76.95±1.67 & 88.65±0.49 & 87.23±1.85 \\
Gen-NSD-Comp & 87.56±3.86 & 89.04±2.93 & 37.72±1.17 & 45.13±2.19 & \underline{47.01±2.55} & 86.49±6.84 & 77.02±1.75 & 89.68±0.31 & 87.82±0.84 \\
Diag-NSD & 85.67±6.95 & 88.63±2.75 & 37.79±1.01 & 45.52±2.22 & 46.55±3.03 & 86.49±7.35 & 77.14±1.85 & 89.42±0.43 & 87.14±1.06 \\
O(d)-NSD & 85.95±5.51 & 89.41±4.74 & 37.81±1.15 & \underline{45.59±2.23} & 46.26±3.11  & 84.86±4.71 & 76.70±1.57 & 89.49±0.40 & 86.90±1.13 \\
Gen-NSD & 82.97±5.13 & 89.21±3.84 & 37.80±1.22 & 45.31±2.05 & 45.60±3.36 & 85.68±6.51 & 76.32±1.65 & 89.33±0.35 & 87.30±1.15 \\
\midrule
DirGnn  & 74.22±3.97 & 71.37±6.57  & 29.30±1.22  & 44.48±1.94 & 45.56±3.36 & 61.46±3.63   & 76.09±1.53 & 85.14±0.44 & 86.20 1.18  \\
SigMaNet & 78.92±4.49 & 80.21±5.07 & 36.59±0.55  & 40.89±1.97 & 40.98±3.88  & 73.53±5.91 & 74.35±0.96 & 88.35±0.64 & 85.51±1.14 \\
MagNet & 79.46±8.13 & 81.18±2.80 & 36.51±0.96 & 41.04±1.84 & 43.82±4.56  & 75.99±5.59 & 77.21±1.69  & 88.47±0.54 & 86.32±1.39 \\ \midrule
GGCN & 84.86±4.55 & 86.86±3.29 & 37.54±1.56 & 40.75±2.44  & 39.71±3.25 & 85.68±6.63 & 77.14±1.45 & 89.15±0.37 & 87.95±1.05 \\
H2GCN & 84.86±7.23 & 87.65±4.98 & 35.70±1.00 & 37.77±1.92  & 42.07±4.13 & 82.70±5.28 & 77.11±1.57 & 89.49±0.38 & 87.87±1.20 \\
GPRGNN & 78.38±4.36 & 82.94±4.21 & 34.63±1.22 & 36.62±2.28& 40.67±2.89 & 80.27±8.11 & 77.13±1.67 & 87.54±0.38 & 87.95±1.18 \\
FAGCN & 82.43±6.89 & 82.94±7.95 & 34.87±1.25 & 41.08±2.27 & 41.90±2.72 & 79.19±9.79 & 77.10±1.81 & 90.21±0.36 & \underline{88.17±1.24} \\
MixHop & 77.84±7.73 & 75.88±4.90 & 32.22±2.34 & 43.46±2.52 & 45.09±3.02 & 73.51±6.34 & 76.26±1.33 & 85.31±0.61 & 87.61±0.85 \\
GCNII & 77.57±3.83 & 80.39±3.40 & 37.44±1.30 & 42.22±2.13  & 43.76±2.49 & 77.86±3.79 & 77.33±1.48 & 90.15±0.43 & \textbf{88.37±1.25} \\
GraphSAGE & 82.43±6.14 & 81.18±5.56 & 34.23±0.99 & 39.22±1.20 & 41.67±2.52 & 75.95±5.01 & 76.04±1.30 & 88.45±0.50 & 86.90±1.04 \\
GCN & 55.14±5.16 & 51.76±3.06 & 27.32±1.10 & 39.47±1.47 & 40.89±4.12  & 60.54±5.30 & 76.50±1.36 & 88.42±0.50 & 86.98±1.27 \\
GAT & 52.16±6.63 & 49.41±4.09 & 27.44±0.89 & 35.62±2.06 & 39.21±3.08 & 61.89±5.05 & 76.55±1.23 & 87.30±1.10 & 86.33±0.48 \\
MLP & 80.81±4.75 & 85.29±3.31 & 36.53±0.70 & 40.45±1.41 & 42.79±3.80 & 81.89±6.40 & 74.02±1.90 & 87.16±0.37 & 75.69±2.00 \\
\bottomrule
\end{tabular}
\label{tab:gnn_benchmark}
\end{table*}

%
The results are reported in Table~\ref{tab:gnn_benchmark}. \name{} consistently outperforms the second-best baseline on 7 out of 9 datasets (also in highly heterophilic cases), obtaining an overall improvement of 1.29\% with the largest improvement of 3.46\% on {\tt Citeseer}.
When compared to SNN methods (NSD and NSD-Compl, which do not incorporate a notion of direction), the directional component of DSNN leads to a strictly stronger performance on 7 datasets out of 9.
Compared with GNN methods which incorporate a notion of direction (DirGNN, SigMaNet and MagNet), the directional component of DSNN leads to a better performance across all the 9 dataset, with an average improvement of approximately 6\% and a largest gain, obtained on {\tt Cornell}, of 15.59\%.
These results show that, by relying on both the expressivity of a cellular sheaf as well as on a notion of directionality, DSNN achieves a more expressive propagation of directional signals. Thanks to this, it manages to outperform both SNN and GNN methods.

\subsection{Synthetic Datasets}\label{sec:synthetic}

To further investigate the role of directionality in the context of Sheaf Neural Networks, we compare \name{} against NSD and NSD-Compl on a set of synthetic graphs generated using the Direct Stochastic Block Model (DSBM).
These datasets are generated by varying: \textit{i)} the number of nodes $n$; \textit{ii)} the number of clusters $C$; \textit{iii)} the probability $\alpha_{ij}$ to create an undirected edge between nodes $i$ and $j$ belonging to different clusters; \textit{iv)} the probability $\alpha_{ii}$ to create an undirected edge between two nodes in the same cluster, and \textit{v)} the probability $\beta_{ij}$ of an edge taking a certain direction.
%
For this experiment, the DBSM datasets are generated with $n= 2500$, $C = 5$, $\alpha_{ii} = 0.1$, $\beta_{ij} = 0.2$, with an increasing inter-cluster density $ \alpha_{ij} \in \{0.05, 0.08, 0.1\}$.
We employ unit-dimensional feature vectors defined as each node's sum of in- and out-degrees.
%
We run the experiments 10 times per dataset with a 80\%/5\%/15\% training/validation/testing split, and report the mean accuracy and standard deviation.


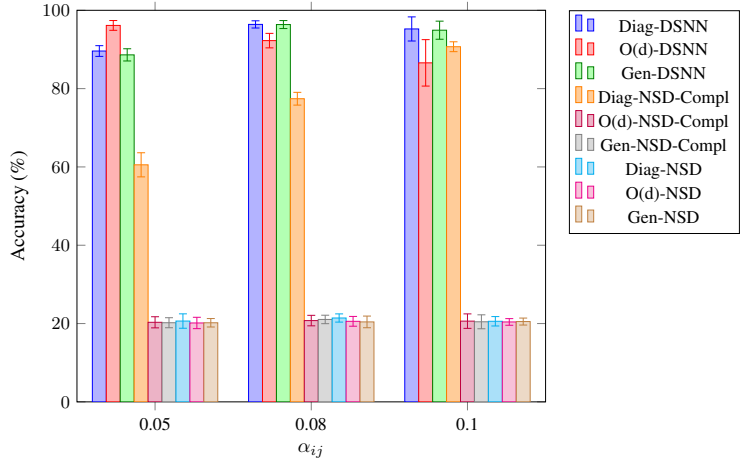
\begin{figure}[]
\centering
\begin{adjustbox}{max width=0.7\linewidth}
\begin{tikzpicture}
\begin{axis}[
    height=8.5cm,
    ybar,
    bar width=7pt,
    ylabel={Accuracy (\%)},
    xlabel={$\alpha_{ij}$},
    symbolic x coords={0.05, 0.08, 0.1},
    xtick=data,
    ymin=0, ymax=100,
    enlarge x limits=0.25,
    yticklabel style={font=\small},
    xticklabel style={font=\small},
    legend style={
      font=\footnotesize,
      at={(1.05,1)},
      anchor=north west,
    },
    error bars/y dir=both,
    error bars/y explicit,
    cycle list={
      {blue,        fill=blue!30,          bar shift=-28pt},
      {red,         fill=red!30,           bar shift=-21pt},
      {green!50!black, fill=green!30,      bar shift=-14pt},
      {orange,      fill=orange!40,        bar shift=-7pt},
      {purple,      fill=purple!30,        bar shift=0pt},
      {gray,        fill=gray!30,          bar shift=7pt},
      {cyan,        fill=cyan!30,          bar shift=14pt},
      {magenta,     fill=magenta!30,       bar shift=21pt},
      {brown,       fill=brown!30,         bar shift=28pt}
    },
]

\addplot+ coordinates {
    (0.05,89.6)  +- (0,1.37)
    (0.08,96.4)  +- (0,0.92)
    (0.1,95.24)  +- (0,3.08)
}; \addlegendentry{Diag-\name}

\addplot+ coordinates {
    (0.05,96.14) +- (0,1.25)
    (0.08,92.26) +- (0, 1.85)
    (0.1,86.58)  +- (0, 5.93)
}; \addlegendentry{O(d)-\name}

\addplot+ coordinates {
    (0.05,88.62) +- (0,1.56)
    (0.08,96.36) +- (0,1.02)
    (0.1,94.92)  +- (0,2.31)
}; \addlegendentry{Gen-\name}

\addplot+ coordinates {
    (0.05,60.54) +- (0,3.08)
    (0.08,77.42) +- (0,1.63)
    (0.1,90.70)  +- (0,1.26)
}; \addlegendentry{Diag-NSD-Compl}

\addplot+ coordinates {
    (0.05,20.32) +- (0,1.42)
    (0.08,20.76) +- (0,1.34)
    (0.1,20.62)  +- (0,1.85)
}; \addlegendentry{O(d)-NSD-Compl}

\addplot+ coordinates {
    (0.05,20.21) +- (0,1.28)
    (0.08,21.05) +- (0,1.08)
    (0.1,20.45)  +- (0,1.78)
}; \addlegendentry{Gen-NSD-Compl}

\addplot+ coordinates {
    (0.05,20.64) +- (0,1.84)
    (0.08,21.42) +- (0,1.05)
    (0.1,20.58)  +- (0,1.20)
}; \addlegendentry{Diag-NSD}

\addplot+ coordinates {
    (0.05,20.15) +- (0,1.45)
    (0.08,20.57) +- (0,1.25)
    (0.1,20.41)  +- (0,0.85)
}; \addlegendentry{O(d)-NSD}

\addplot+ coordinates {
    (0.05,20.20) +- (0, 1.08)
    (0.08,20.42) +- (0,1.49)
    (0.1,20.51)  +- (0,0.89)
}; \addlegendentry{Gen-NSD}

\end{axis}
\end{tikzpicture}
\end{adjustbox}
\caption{Mean and standard deviation of the accuracy on a collection of synthetic benchmarks on the node classification task as a function of the parameter $\alpha_{ij}$ used during graph generation.}
\label{fig:fig2}
\end{figure}

As shown in Figure~\ref{fig:fig2}, the three variants of \name---diagonal, O(d), and general---which rely on the {Directed Sheaf Laplacian} substantially outperform both NSD and NSD-Compl.
The three variants of DSNN consistently achieve an average performance between 86\% and 96\% across all three settings, one per value of $\alpha_{ij}$, whereas all the other baselines except for one (Diag-NSD-Compl) achieve an average accuracy of 20\% which, due to having $C=5$ classes, is as poor as the accuracy of trivial uniform predictor.
The better performance achieved by Diag-NSD-Compl over the other baselines (but not DSNN) can be attributed to its more limited number of trainable parameters compared to the other NSD-Compl variants, which facilitates convergence and enables the model to partially capture some directional interactions in its complex-valued (albeit classical) Sheaf Laplacian.
Nonetheless, its performance remains substantially lower than DSNN's by 29\% on average.
%
These results underscore \name’s ability to leverage edge directionality.
They further substantiate that the way we defined our complex-valued Laplacian matrix by embedding topological information via the imaginary part of the restriction maps leads to a substantially better performance that what one would obtain by simply learning complex-valued restriction maps end-to-end without relying on an inductive bias coming from the graph's topology (which is what is done in NSD-Compl).
%
%






\section{Conclusion}

{\bf Summary and findings.} We introduced the Directed Cellular Sheaf, from which we derived the Directed Sheaf Laplacian $\Lc$. By encoding the edge direction in its imaginary components, 
$\Lc$ carries a directional inductive bias thanks to which we obtain a convolution operator implementing a message-passing scheme capable of handling asymmetric interactions. We embedded such an operator in the Directed Sheaf Neural Network (DSNN).
Our theoretical results have shown that DSNN generalizes several well-established graph-learning models, including NSD, MagNet, and SigMaNet.
Empirically, DSNN exhibits strong performance across both real-world and synthetic datasets, consistently outperforming both traditional GNNs and SNNs. This demonstrates that DSNN's explicit treatment of directionality leads to superior generalization, particularly in heterophilic graph settings.

{\bf Limitations and future works}.
%
Further work may include considering node-varying stalk sizes as well as an extended assessment of DSNN on datasets and tasks arising from various real-world applications, e.g., life science and healthcare. Further work may also address low-level optimization aspects of our pipeline to potentially reduce its computational footprint, allowing for its adoption in low-power and low-compute settings.






\bibliographystyle{unsrt}
\bibliography{BIB}

\newpage
\appendix

\section{Licensing}\label{appx:implementation}


%
%
The {\tt Cora}, {\tt Citeseer}, and {\tt PubMed} datasets are available at~\url{https://linqs.org/datasets/}
{\tt Citeseer}, {\tt PubMed}, and {\tt Cora} {are sourced from~\citep{yan2022two}}.
The {\tt Squirrel} and {\tt Chameleon} datasets come from~\cite{platonovcritical}.

Regarding the models used in this paper, we rely on publicly available implementations from the following sources:
\begin{itemize}
    \item \textbf{MLP, GCN, GAT, GGCN, GCNII, Geom-GCN, GPRGNN:} \url{https://github.com/Yujun-Yan/Heterophily_and_oversmoothing} with MIT license.
    \item \textbf{GraphSAGE:} \url{https://pytorch-geometric.readthedocs.io/en/latest/generated/torch_geometric.nn.conv.SAGEConv.html} with MIT license.
    \item \textbf{H2GCN:} \url{https://github.com/Godofnothing/HeterophilySpecificModels/tree/main/H2GCN}.
    \item \textbf{FAGCN:} \url{https://github.com/Godofnothing/HeterophilySpecificModels/tree/main/FAGCN}.
    \item \textbf{MixHop:} \url{https://github.com/benedekrozemberczki/MixHop-and-N-GCN} with GNU General Public License v3.0 (GPL-3.0) license.
    \item \textbf{MagNet:} \url{https://github.com/matthew-hirn/magnet} with Apache License 2.0.
    \item \textbf{SigMaNet:} \url{https://github.com/Stefa1994/SigMaNet} with Apache License 2.0.
    \item \textbf{DirGNN:} \url{https://github.com/emalgorithm/directed-graph-neural-network} with Apache License 2.0.
    \item \textbf{NSD:} \url{https://github.com/twitter-research/neural-sheaf-diffusion} with Apache License 2.0.
\end{itemize}

\section{Derivation of the equation of $\Lc$}

Since, by construction, $\Lc = \deltac^* \deltac$, the following equation holds:
\begin{eqnarray*}
    \Lc(x)_u =  \underbrace{\sum_{e \in \Gamma(u)} \Fc_{u \trianglelefteq e}^* \Fc_{u \trianglelefteq e} x_u}_{\text{self-loop}}
    -  \underbrace{\sum_{e=(v,u) \in E} (\F_{u \trianglelefteq e}^0 T_{\red{vu}}^{(q)})^* \F_{v \trianglelefteq e} x_v}_{\text{inflow}}\\
    - \underbrace{\sum_{e=(u,v) \in E} \F_{u \trianglelefteq e}^T  \F_{v \trianglelefteq e}^0 T_{\red{uv}}^{(q)}x_v}_{\text{outflow}}
    - \underbrace{\sum_{e=\{u,v\} \in E} \F_{u \trianglelefteq e}^T \F_{v \trianglelefteq e}x_v}_{\text{undirected}}. 
\end{eqnarray*}

For a directed graph $G$, the $uu$ component of $\Lc$ can be rewritten as follows:
\begin{align*}
\Lc_{uu} = \sum_{e \in \Gamma(u)} \Fc_{u \trianglelefteq e}^* \Fc_{u \trianglelefteq e}
= & \sum_{e=(u,v) \in E} \Fc_{u \unlhd e}^T \Fc_{u \unlhd e}\\
+ & \sum_{e=(v,u) \in E} (\Fc_{u \unlhd e}^0)^T(T_{vu}^{(q)})^* \Fc_{u \unlhd e}^0 T_{vu}^{(q)}
+  \sum_{e=\{u,v\} \in E} \Fc_{u \unlhd e}^T \Fc_{u \unlhd e}.
\end{align*}
%

Equation~\eqref{eq:Lc} is obtained by combining this equation with the previous one and factoring each summation by $(\Fc_{u \trianglelefteq e})^T$.

\section{Proofs of our theorems}\label{appx:proof}

\setcounter{theorem}{0}

\begin{theorem}
  $\Lc$ is Hermitian and $\Lc \succeq 0$, and the same holds for $\Lc_N$.
\end{theorem}
\begin{proof}
By definition, we have $\Lc := \deltac^* \deltac$. Therefore, for any pair of indices $u,v \in V$, $\Lc_{uv} = \deltac_{\bullet u}^* \deltac_{\bullet v}$  and $\Lc_{vu} = \deltac_{\bullet v}^* \deltac_{\bullet u}$.  Since this implies $\Lc_{uv} =  (\Lc_{vu})^*$ for all of its entries $u,v$, we conclude that the matrix is Hermitian. It follows that the spectrum of $\Lc$ is real. By following essentially the same arguments using $\deltac' := \deltac D^{-\frac{1}{2}}$, one can show that the spectrum of $\Lc_N$ is real as well.

Since, again by definition, $\Lc = \deltac^* \deltac$, its associated quadratic form $x^* \deltac^* \deltac x$ (with $x \in \mathbb{C}$) can be rewritten as $x^* \deltac^* \deltac x = (\deltac x)^* (\deltac x) = || \deltac x||_2^2$. Since $||\deltac x||_2^2$ is a norm, $||\deltac x||_2^2 \geq 0$ holds for all $x \in \mathbb{C}$, thus implying $\Lc \succeq 0$ for all $x^* \in \mathbb{C}$. Thus, $\Lc \succeq 0$.  By following the same arguments using $\deltac' := \deltac D^{-\frac{1}{2}}$, one can show that $\Lc_N \succeq 0$ as well.
\end{proof}

\begin{theorem}
  $\Deltac \preceq 2 I$.
\end{theorem}
\begin{proof}
  Let $Q^{\Fc}_N := D^{-\frac{1}{2}}\deltac^* \deltac D^{-\frac{1}{2}}$ for the case where $\deltac$ has {\em not} been given an arbitrary orientation (this is in line with the classical construction of the Signless Laplacian $Q$ for undirected unweighted graphs).
  Since $Q^{\Fc}_N$ is  the product of a matrix and its conjugate, we have $Q^{\Fc}_N \succeq 0$.
  It is easy to show that $Q^{\Fc}_N = 2I - \hat L^\F_N$. From this, we deduce:
    $$
    Q^{\Fc}_N = 2I - \Lc_N \succeq 0 \Leftrightarrow - \Lc_N \succeq -2 I \Leftrightarrow \Lc_N \preceq 2I.
    $$
    This shows that not only $\Lc_N$ has a nonnegative spectrum, but also that its spectrum is upper-bounded by $2$.
\end{proof}

\begin{theorem}
  If $G$ is undirected, $\Lc$ coincides with the classical sheaf Laplacian $L^\F$ for any choice of $q \in \R$. Also, if the   sheaf is trivial and $G$ is undirected and unweighted, $\Lc$ coincides with the classical graph Laplacian $L$. If $G$ is directed and we set $q = 0$, $\Lc$ coincides with the classical sheaf Laplacian associated with the undirected version of $G$.
\end{theorem}
\begin{proof}
{\bf Part 1.} If $G$ is undirected, all restriction maps of the Directed Cellular Sheaf are real for every choice of $q \in \R$--this is because, for all $u,v \in V$, $A = A^T$ implies $\Re(T_{uv}^{(q)}) = 1$ and $\Im(T_{uv}^{(q)}) = 0$ for any choice of $q$. This implies $\Im(\Fc_{u \trianglelefteq e}) = 0$ for all $e \in E$ where $u$ is one of its endpoints; therefore, $\Lc$ is real valued and $\Lc = L^\F$.\\
{\bf Part 2.}
Under the same assumptions on $G$, if the Directed Cellular Sheaf is trivial, $d = 1$ and $\Fc_{u \trianglelefteq e} = 1$ for all edges $e \in E$ with $u$ being one if its endpoints. Thus, $\Lc_{uv} = -1$ if $\{u,v\} \in E$ and 0 otherwise, while $\Lc_{uu} = |\{e \in E: e = \{u,v\}\}|$; by definition, it follows that $\Lc$ coincides with the classical Laplacian matrix $L = D-A$ with $A \in \{0,1\}^{n \times n}$.\\
{\bf Part 3.} Setting $q = 0$ leads to, for all $u,v \in V$, $T_{uv}^{(q)} = \cos(0) + \ii \sin(0)= 1$. Thus,
$\Lc$ coincides with the Directed Sheaf Laplacian $L^\F$ associated with the undirected version of $G$ which is obtained from it by preserving each of its edges and making all of them undirected---this coincides with discarding $\Im(\Fc_{u \trianglelefteq e}) = 0$ for all $e \in E$ where $u$ is one of its endpoints.
\end{proof}

\begin{theorem}
  Letting $G$ be a directed graph with unit weights, the Directed Sheaf Laplacian $\Lc$ associated with a Trivial Directed Cellular Sheaf coincides with the Magnetic Laplacian $L^{(q)}$. In the special case where $q=\frac{1}{4}$, $\Lc$ also coincides with the Sign-Magnetic Laplacian $L^\sigma$.
\end{theorem}
\begin{proof}
{\bf Part 1.} First, we show that, when adopting a Trivial Directed Cellular Sheaf for a directed graph $G$ with unit weights, we have:
\begin{align*}
\Lc_{uv} = -T_{uv}^{(q)} & \qquad u,v \in V: u \neq v\\
\Lc_{uu} = |\Gamma(u)| & \qquad u \in V.
\end{align*}

Eq.~\ref{eq:Lc:uv} and~\ref{eq:Lc:uu} read:
\begin{align*}
\Lc_{uv} =& \left\{\begin{array}{ll}
                    -\Fc_{u \unlhd e}^* \Fc_{v \unlhd e} = -\Fc_{u \unlhd e}^T \Fc_{v \unlhd e}^0 T_{\red{uv}}^{(q)} & \text{ if } e=(u,v)\\
                    -\Fc_{u \unlhd e}^* \Fc_{v \unlhd e} = -(\Fc_{u \unlhd e}^0 T_{\red{vu}}^{(q)})^* \Fc_{v \unlhd e}  & \text{ if } e=(v,u)\\
                    -\Fc_{u \unlhd e}^* \Fc_{v \unlhd e} = -\Fc_{u \unlhd e}^T \Fc_{v \unlhd e} & \text{ if } e=\{u,v\}\\
                    0 & \text{ otherwise}
\end{array}
  \right.\\
  \Lc_{uu} =& \sum_{e \in \Gamma(u)} \Fc_{u \trianglelefteq e}^* \Fc_{u \trianglelefteq e}
  .
\end{align*}
When considering a Trivial Directed Cellular Sheaf, we have
\begin{itemize}
    \item $\Fc_{u \trianglelefteq e} = \Fc_{v \trianglelefteq e} = 1$ if $e=\{u,v\} \in E$ and, thus, $\Lc_{uv}= -1 = -T_{uv}^{(q)}$ (th latter is because $A_{uv} = A_{vu}$ implies $T_{uv}^{(q)} = \cos(0) + \ii \sin(0) = 1$).
    \item $\Fc_{u \trianglelefteq e} = 1$ and $ \Fc_{v \trianglelefteq e} = T_{\red{uv}}^{(q)}$ if $e=(u,v) \in E$ and, thus, $\Lc_{uv}= -T_{uv}^{(q)}$;
    \item $\Fc_{u \trianglelefteq e} = T_{\red{vu}}^{(q)}$ and $\Fc_{v \trianglelefteq e} = 1$ if $e=(v,u) \in E$ and, thus, $\Lc_{uv}= -(T_{vu}^{(q)})^* = -T_{uv}^{(q)}$.
\end{itemize}

Each diagonal term $\Lc_{uu}$ of $\Lc$ reads
\begin{align*}
\Lc_{uu} = \sum_{e \in \Gamma(u)} \Fc_{u \trianglelefteq e}^* \Fc_{u \trianglelefteq e}
= & \sum_{e=(u,v) \in E} \underbrace{\Fc_{u \unlhd e}^T \Fc_{u \unlhd e}}_{=1}\\
+ & \sum_{e=(v,u) \in E} \underbrace{(\Fc_{u \unlhd e}^0)^T(T_{vu}^{(q)})^* \Fc_{u \unlhd e}^0 T_{vu}^{(q)}}_{= (T_{vu}^{(q)})^* (T_{vu}^{(q)}) = 1}
+  \sum_{e=\{u,v\} \in E} \underbrace{\Fc_{u \unlhd e}^T \Fc_{u \unlhd e}}_{=1} \\
& =|\Gamma(u)|,
\end{align*}
where $(T_{vu}^{(q)})^* (T_{vu}^{(q)}) = 1$ holds since $T_{vu}^{(q)} = \ii$.
%
With this, Part 1 is shown.

{\bf Part 2.}

The Magnetic Laplacian reads
\begin{equation*}
 L^{(q)} := D_s - H^{(q)}, \text{ with } H^{(q)} := A_s \odot \exp \left(\ii \, 2 \pi q\left(A-A^\top \right) \right),
\end{equation*}
with $A_s := \frac{A + A^T}{2}$ and $D_s = \diag(\mathbf{1}_n A_s)$.

By definition we gave of $T_{uv}^{(q)}$, for a component $u,v$ with $u,v \in V$, we have:
\begin{equation*}
 L^{(q)}_{uv} := D_{s_{uv}} - H^{(q)}_{uv} = D_{s_{uv}} - A_{s_{uv}}  T_{uv}^{(q)}.
\end{equation*}

{\bf Part 2a.} Let's assume $G$ undirected.
In such a case, we have we have $A_{s_{uv}} = 1$ whenever $\{u,v\} \in E$ and $A_{s_{uv}} = 0$ otherwise. This implies $D_{s_{uu}} = |\Gamma(u)|$.
Thus, we have:
\begin{align*}
 L^{(q)}_{uv} = & - A_{s_{uv}}  T_{uv}^{(q)} = -T_{uv}^{(q)} = \Lc_{uv} & \qquad u,v \in V: u \neq v\\
 L^{(q)}_{uu} = & D_{s_{uu}} - A_{s_{uu}}  T_{uu}^{(q)} = D_{s_{uu}} = |\Gamma(u)| = \Lc_{uu} & \qquad u \in V,
\end{align*}
where the last equation holds since $T_{uu}^{(q)} = 0$ for any $q$. Thus, $\Lc = L^{(q)}$.

{\bf Part 2b.} Let's assume $G$ directed without digons.
In such a case, we have $A_{s_{uv}} = \frac{1}{2}$ whenever either $(u,v) \in E$ or $(v,u) \in E$ and $A_{s_{uv}} = 0$ otherwise. This implies $D_{s_{uu}} = \frac{1}{2}|\Gamma(u)|$.
Thus, we have:
\begin{align*}
 L^{(q)}_{uv} = & - A_{s_{uv}}  T_{uv}^{(q)} = -\frac{1}{2}T_{uv}^{(q)} = \frac{1}{2}\Lc_{uv} & \qquad u,v \in V: u\neq V\\
 L^{(q)}_{uu} = & D_{s_{uu}} - A_{s_{uu}}  T_{uu}^{(q)} = D_{s_{uu}} = \frac{1}{2}|\Gamma(u)| = \frac{1}{2}\Lc_{uu} & \qquad u \in V,
\end{align*}
where the last equation holds since $T_{uu}^{(q)} = 0$ for any $q$. Thus, $\Lc = 2L^{(q)}$. Notice that the scaling factor is immaterial when the Laplacian matrix is embedded in a GCN/SNN, as it is directly subsumed by either $W_1$ or $W_2$ in Eq~\eqref{eq:sheafneural} (only by the latter in a GCN, where $W_1$ is not present).

{\bf Part 3.} Since, as shown in~\cite{fiorini2023sigmanet}, $L^{(q)}$ and $L^{\sigma}$ coincide with $q = \frac{1}{4}$, the last part of the claim follows directly from Parts 2a and 2b.
\end{proof}

\begin{theorem}
    Let $G$ be a directed graph with unit weights. \red{Assuming a Trivial Directed Cellular Sheaf}, the conjugate transpose $\deltac^*$ of the Directed Coboundary Operator $\deltac$ boils down to the complex-valued node-to-edge incidence matrix $\widehat B \in \mathbb{C}^{\red{n \times m}}$ defined for an edge $e \in E$ incident to a vertex $u$:
  \begin{equation*}
    \widehat B_{\red{ue}} = \left\{\begin{array}{ll}
                    1 & \text{ if } e=(u,v) \text{ or } e=\{u,v\} \text{ with } \red{u < v}\\
                    \red{-1} & \text{ if } e=\{u,v\} \text{ with } \red{u > v}\\
                    -T_{uv}^{(q)} & \text{ if } e=(v,u).
                  \end{array}
  \right.
  \end{equation*}
  It follows that $L^{(q)} = \widehat B \widehat B^*$.
  With $q=\frac{1}{4}$, $L^{(\frac{1}{4})} = L^{\sigma} = \widehat B \widehat B^*$.
\end{theorem}
\begin{proof}
(First, notice the arbitrary orientation that was given to the undirected edges).

From the proof of the previous theorem, we know that, if $G$ has unit weights and the Directed Cellular Sheaf is trivial, we have:
\begin{align*}
\Lc_{uv} = -T_{uv}^{(q)} & \qquad u,v \in V: u \neq v\\
\Lc_{uu} = |\Gamma(u)| & \qquad u \in V.
\end{align*}

Let's consider $(\widehat B \widehat B^*)_{uv} = \sum_{e' \in E} \widehat B_{u e'} (\widehat B_{v e'})^*$. Since we are considering a graph, $u,v$ can only share a single edge. Calling it $e$, we have $(\widehat B \widehat B^*)_{uv} = \widehat B_{u e} (\widehat B_{v e})^*$ if $e \in E$ or $0$ if they share no edge at all. Let's assume they do, and considering three cases:
\begin{itemize}
 \item If $e =\{u,v\}$ with $u < v$, $\widehat B_{u e} = 1$ and $(\widehat B_{v e})^* = -1$ with an arbitrary orientation and, thus, $\widehat B_{u e} (\widehat B_{v e})^* = - 1 = -T_{uv}^{(q)}$ (this is correct since, as shown before, $T_{uv}^{(q)}$ is always equal to 1 if $A_{uv} = A_{vu}$).
 \item If $e =\{u,v\}$ with $u > v$, $\widehat B_{u e} = -1$ and $(\widehat B_{v e})^* = 1$ with an arbitrary orientation and, thus, $\widehat B_{u e} (\widehat B_{v e})^* = - 1 = -T_{uv}^{(q)}$ (as shown before, the latter is always equal to 1 if $A_{uv} = A_{vu}$).
\item If $e =(u,v)$, $\widehat B_{u e} = 1$ and $(\widehat B_{v e})^* = (-T_{vu}^{(q)})^*$ and, thus, $\widehat B_{u e} (\widehat B_{v e})^* = (-T_{vu}^{(q)})^* = -T_{uv}^{(q)}$ since $T^{(q)}$ is Hermitian by construction.
\item If $e =(v,u)$, $\widehat B_{u e} = -T_{uv}^{(q)}$ and $(\widehat B_{v e})^* = 1$ and, thus, $\widehat B_{u e} (\widehat B_{v e})^* = -T_{uv}^{(q)}$.
\end{itemize}
This shows that, if $G$ has unit weights and assuming a Trivial Directed Cellular Sheaf, we have $\Lc = \widehat B \widehat B^*$.
The fact that (with a scaling factor of 2, when needed) $L^{(q)} = \widehat B \widehat B^*$ and $L^{\sigma} = \widehat B \widehat B^*$ when $q=\frac{1}{4}$ follow from the previous theorem.
\end{proof}

\section{Complexity of \name{}}\label{appx:complexity}

As mentioned in the paper, the complexity of \name{} coincides, asymptotically, with that of NSD. This is because the adoption of complex-valued restriction maps---which are specific to \name{} and not present in NSD---does not affect the asymptotic inference complexity of \name{}. This is because complex-valued synaptic weights, pre-activations, and activations only incur a constant multiplicative overhead (approximately a factor of 4) in the forward pass and, thus, do not alter the asymptotic complexity from the real-valued case analysis.
To better see this, consider three complex-valued matrices:
\[
A = A_R + \ii A_I, \quad X = X_R + \ii X_I, \quad Y = Y_R + \ii Y_I,
\]
with
\[
A_R, A_I \in \mathbb{R}^{m \times n}, \quad X_R, X_I \in \mathbb{R}^{n \times p}, \quad Y_R, Y_I \in \mathbb{R}^{m \times p},
\]
satisfying the complex linear equation \( Y = AX \).
This equation can be rewritten purely in the real domain using the \emph{lifting} transformation:
\[
X_\mathbb{R} =
\begin{bmatrix}
X_R \\
X_I
\end{bmatrix}
\in \mathbb{R}^{2n \times p}, \quad
Y_\mathbb{R} =
\begin{bmatrix}
Y_R \\
Y_I
\end{bmatrix}
\in \mathbb{R}^{2m \times p}, \quad
A_\mathbb{R} =
\begin{bmatrix}
A_R & -A_I \\
A_I & A_R
\end{bmatrix}
\in \mathbb{R}^{2m \times 2n},
\]
so that$
Y_\mathbb{R} = A_\mathbb{R} X_\mathbb{R}$ holds.
Hence, complex-valued operations can be reduced to real-valued operations with a constant factor overhead, which is immaterial in the asymptotic complexity.

\section{Further Details on the Datasets}\label{appx:dataset}

\paragraph{Real-world dataset.}

The {\tt Texas}, {\tt Wisconsin}, and {\tt Cornell} datasets are part of the WebKB collection, modeling links between websites from different universities. In these datasets, nodes are labeled as student, project, course, staff, or faculty.

The {\tt Film} dataset is derived from a film–director–actor–writer network. Each node represents an actor, and edges indicate co-occurrence on the same Wikipedia page. Node features correspond to keywords extracted from these Wikipedia pages. The nodes are classified into five categories based on the content of the actors’ Wikipedia entries.

The {\tt Citeseer} dataset contains 3,312 scientific publications classified into six categories. The citation network includes 4,732 links. Each publication is represented by a binary word vector indicating the presence or absence of words from a dictionary of 3,703 unique terms.

The {\tt PubMed} dataset consists of 19,717 scientific publications related to diabetes, categorized into three classes. The citation network contains 44,338 links. Each publication is described by a TF-IDF weighted word vector derived from a dictionary of 500 unique words.

The {\tt Cora} dataset includes 2,708 scientific publications classified into seven classes, with a citation network comprising 5,429 links. Each publication is represented by a binary word vector indicating the presence or absence of words from a dictionary of 1,433 unique terms.

The {\tt Squirrel} and {\tt Chameleon} datasets consist of articles from the English Wikipedia (December 2018). Nodes represent articles, and edges represent mutual links between them. Node features indicate the presence of specific nouns in the articles. Nodes are grouped into five categories based on the original regression targets.

\paragraph{Synthetic dataset.} 
Following~\cite{zhang2021magnet}, we generate synthetic graphs using the directed stochastic block model (DSBM) as follows. Let $n$ be the number of nodes and $C$ the number of equal‐sized communities $\{C_1,\dots,C_C\}$. First, we sample an undirected graph by connecting each pair of nodes $u\in C_i$ and $v\in C_j$ independently with probability
$
\alpha_{ij} \in [0,1],
\quad
\alpha_{ij} = \alpha_{ji},
$
where $\alpha_{ii}$ controls intra‐community edge density and $\alpha_{ij}$ for $i \neq j$ controls inter‐community connectivity.
To obtain a directed graph, we introduce a rule to transform the graph from undirected to directed: we define a collection of probabilities $\{\beta_{ij}\}_{1\le i,j\le C}$, where $
\beta_{ij} \in [0,1]$, such that 
$\beta_{ij} + \beta_{ji} = 1.
$
If $u\in C_i$ and $v\in C_j$, we orient the edge $u \to v$ with probability $\beta_{ij}$, and $v \to u$ with probability $\beta_{ji}$.

\section{Further Details on the Experiments}\label{appx:experiment}


\paragraph{Hardware.} The experiments were conducted on 2 different machines: 
\begin{enumerate}
    \item An Intel(R) Xeon(R) Gold 6326 CPU @ 2.90GHz with 380 GB RAM, equipped with an NVIDIA Ampere A100 40GB.
    \item A 12th Gen Intel(R) Core(TM) i9-12900KF CPU @ 3.20GHz CPU with 64 GB RAM, equipped with an NVIDIA RTX 4090 GPU.
\end{enumerate}

\paragraph{Model Settings.} We trained every learning model considered in this paper for up to 1000 epochs with early stops of 200. We adopted a learning rate of $\{1 \cdot 10^{-2}, 2 \cdot 10^{-2}, 5 \cdot 10^{-3}\} $ and employed the optimization algorithm Adam. 

We adopted a hyperparameter optimization procedure to identify the best set of parameters for each model. 
For every model, we searched for the optimal combination of the following hyperparameters:
\begin{itemize}
    \item \textbf{Dropout:} \{0.0, 0.1, 0.2, 0.3, 0.4, 0.5, 0.6, 0.7, 0.8, 0.9\}
    \item \textbf{Number of layers:} \{2, 3, 4, 5, 6\}
    \item \textbf{Hidden channels:} \{8, 16, 32, 64, 128\}.
\end{itemize}

For some specific models, we also included additional hyperparameters in the search space:
\begin{itemize}
    \item \textbf{NSD-comp and NSD:} \texttt{sheaf\_act} $\in$ \{\texttt{elu}, \texttt{tanh}, \texttt{relu}\}; $d \in \{2, 3, 4, 5\}$; \texttt{add\_lp} $\in$ \{\texttt{True}, \texttt{False}\}; \texttt{add\_hp} $\in$ \{\texttt{True}, \texttt{False}\}
    \item \textbf{DirGNN:} $\alpha_{\text{DirGNN}} \in$ \{0.0, 0.1, 0.2, 0.3, 0.4, 0.5, 0.6, 0.7, 0.8, 0.9, 1.0\}; \texttt{jk} $\in$ \{\texttt{cat}, \texttt{max}\}
    \item \textbf{MagNet:} $q \in$ \{0.0, 0.05, 0.1, 0.15, 0.2, 0.25\}
    \item \textbf{GCNII:} $\alpha_{\text{GCNII}} \in $ \{0.0, 0.1, 0.2\}; $\lambda \in$ \{0.0, 1.0, 1.5\}
    \item \textbf{FAGCN:} $\varepsilon \in$ \{0.0, 0.1, 0.2, 0.3, 0.4, 0.5, 0.6, 0.7, 0.8, 0.9, 1.0\}
    \item \textbf{GGCN:} \texttt{decay\_rate} $\in$ \{0.0, 0.1, 0.2, 0.3, 0.4, 0.5, 0.6, 0.7, 0.8, 0.9, 1.0, 1.1, 1.2\}
    \item \textbf{GPRGNN:} $\alpha_\text{GPRGNN}$ $\in$ \{0.0, 0.1, 0.2, 0.3, 0.4, 0.5, 0.6, 0.7, 0.8, 0.9, 1.0\}; \texttt{dprate\_GPRGNN} $\in$ \{0.0, 0.1, 0.2, 0.3, 0.4, 0.5, 0.6, 0.7, 0.8, 0.9\}
    \item \textbf{\name:} $q \in [ 0, 0.1, 0.15, 0.20, 0.25,0.5,0.75, 1 ]$, \texttt{sheaf\_act} $\in$ \{\texttt{elu}, \texttt{tanh}, \texttt{relu}\}, $d \in \{2, 3, 4, 5\}$; \texttt{add\_lp} $\in$ \{\texttt{True}, \texttt{False}\}; \texttt{add\_hp} $\in$ \{\texttt{True}, \texttt{False}\}
    
\end{itemize}

\end{document}